\documentclass[11pt]{article}
\usepackage[final]{acl}
\usepackage{times}
\usepackage{latexsym}
\usepackage{float}
\usepackage{booktabs}
\usepackage{tcolorbox}
\usepackage{multirow,makecell}
\usepackage{subcaption}
\usepackage{caption}
\usepackage[T1]{fontenc}
\usepackage{xcolor}
\usepackage[utf8]{inputenc}
\usepackage{microtype}
\usepackage{inconsolata}
\usepackage[linesnumbered,ruled,vlined]{algorithm2e}

\usepackage{graphicx}
\usepackage{amsmath}
\usepackage{amssymb}

\usepackage{xcolor}

\newtheorem{theorem}{Theorem}
\newtheorem{lemma}[theorem]{Lemma}

\newtcolorbox{promptbox}[1][]{
  colback=gray!10, colframe=gray!40, arc=5pt, boxrule=1pt,
  fontupper=\small\ttfamily, title=#1
}

\usepackage{xspace}
\newcommand{\name}{TrustPropRAG\xspace}

\title{Feedback-Assisted Trust Propagation over Document Relation Graphs for Retrieval-Augmented Generation}

\author{Zhuoheng Li, Ying Chen \\
College of Information Sciences and Technology, The Pennsylvania State University\\
\texttt{\{zml5515,yingchen\}@psu.edu}
}

\begin{document}
\maketitle
\begin{abstract}
Retrieval-augmented generation (RAG) systems rely on external corpora that may contain outdated, contradictory, noisy, or unreliable documents, introducing reliability risks. Prior work has leveraged document relations to improve the answer reliability of RAG. To propagate reliability signals beyond directly compared document pairs, we propose \mbox{TrustPropRAG}, which structures document relations as a graph and estimates document reliability through multi-hop propagation across the graph.
TrustPropRAG anchors this propagation with a limited set of human feedback on document
reliability, extending these costly-to-collect feedback-based reliability signals across the whole corpus.
Specifically, based on the constructed document relation graph, \mbox{TrustPropRAG} estimates a trust score for each document by formulating and solving an optimization problem that jointly captures pairwise document relations and user feedback. These scores are then used to improve the selection of reliable documents and support trust-aware answer generation. Evaluation results show that TrustPropRAG
improves both retrieval quality and exact match over baselines, and remains robust under sparse and noisy feedback.

\end{abstract}

\section{Introduction}
\label{sec:intro}
Retrieval-augmented generation (RAG) has emerged as an effective approach for mitigating the limitations of large language models (LLMs), whose parametric knowledge may be incomplete or outdated~\cite{lewis2020rag,borgeaud2022retro,izacard2022atlas}. By retrieving knowledge from external corpora, LLMs can incorporate
relevant and up-to-date information during answer generation. However, RAG's reliance on external corpora can introduce reliability risks,
as the corpora may contain noisy, outdated, or contaminated content~\citep{zhong2023poisoning,zou2024poisonedrag}, as well as knowledge conflicts~\citep{xie2024contradictions,chen2022rich}.

Prior work aims
to improve RAG reliability in both the retrieval~\cite{ma2023query} and answer generation~\cite{wang2025astuterag} phases.
Notably, a line of work
leverages relationships among documents to improve the answer reliability of RAG, such as aggregating answers from different documents through majority agreement~\cite{xiang2024certifiably}, or selecting a subset of 
documents that are mutually consistent before answer generation~\cite{shen2025reliabilityrag}. Document relations can be inferred by using natural language inference (NLI) models or leveraging source metadata.

We propose \name{}, a framework that structures document relations as a graph, as graph representations enable information to propagate through \emph{multi-hop relations}. 
Recent studies have explored graph-enhanced RAG, but they mainly focus on organizing knowledge graphs~\cite{edge2024local}, or constructing index graphs for retrieval~\cite{guo2024lightrag}. In contrast, we use the graph structure to estimate document reliability and propagate these estimates across the corpus through multi-hop relations.
Graph propagation, however, requires anchoring, as document relations reveal consistency or conflict, but not document reliability itself. We anchor this propagation with a set of human feedback on document reliability.
In practice, explicit signals  (e.g., thumbs-up or thumbs-down ratings) and implicit signals (e.g., clicks or dwell time) can be attributed to the documents that contribute to the generated answer using traceback methods~\cite{wang2025tracllm,cohenwang2024contextcite}, thereby providing positive or negative signals for those documents.
Since user feedback is costly to collect, TrustPropRAG leverages document relation graphs to propagate sparse feedback across the corpus, similar to how label propagation methods in semi-supervised learning~\citep{zhu2003semisup,zhou2003lgc} spread a few labeled examples across many unlabeled ones.
\name then turns the propagated document reliability scores into performance improvements in retrieval and answer generation.

To estimate document reliability and propagate estimates across the corpus,
\name first constructs a document relation graph, where nodes represent documents and edges capture document relations (e.g., shared source, mutual support, or contradiction). Based on this graph, TrustPropRAG estimates a trust score for each document by formulating and solving an optimization problem. The problem includes two objective function terms to jointly capture pairwise document relations and 
incorporate user feedback.
At query time, these trust scores are used for trust-aware rescoring, which helps select more reliable documents, and for trust-aware prompting, where the LLM is informed of document reliability during answer generation. Trust score optimization is performed offline over the corpus, amortizing its computational cost across future queries.

Our contributions are threefold.
First, we propose a framework that estimates document trust scores by propagating human feedback over a document relation graph, and uses these scores to improve both retrieval and answer generation. Second, we formulate and solve an optimization problem for trust score estimation. We also provide a theoretical analysis of how feedback influences trust score propagation, as well as the convergence rate of the propagation. 
Third, we evaluate \name{} with three open-domain question-answering (QA) benchmarks, three retrievers, and six LLMs, and further test its robustness under sparse and noisy feedback.
Our code is publicly available at~\url{https://github.com/zhliOvO/TrustPropRAG}.

\section{Related Work}
\label{sec:related}

\noindent \textbf{RAG reliability.} 
A growing body of work has studied how to improve RAG reliability. For example, InstructRAG guides LLMs to denoise retrieved contexts using rationales~\cite{wei2024instructrag}. AstuteRAG compares retrieved evidence with the LLM’s parametric knowledge to handle knowledge conflicts~\citep{wang2025astuterag}. A notable line of work uses relationships (agreements or contradictions) among documents to improve reliability. For instance, RobustRAG isolates retrieved documents, generates per-document answers independently, and aggregates them through majority agreement~\cite{xiang2024certifiably}. TrustRAG leverages similarity relations among retrieved documents by clustering them and then filtering out contaminated documents before generation~\cite{zhou2025trustrag}. 
ReliabilityRAG detects contradictions between documents, selects a subset of documents that are mutually consistent, and produces the final answer via keyword aggregation~\citep{shen2025reliabilityrag}.
 However, these relationships are considered within a subset of documents, and reliability is not propagated beyond the directly compared documents. To move beyond this one-hop relation, TrustPropRAG structures document relations as a graph and captures multi-hop relations among documents through the graph representation.  
 
\noindent \textbf{Graph-empowered RAG.}
Graphs have long been used to model relationships among textual evidence in tasks such as multi-document summarization~\citep{christensen2013towards} and fact verification~\citep{zhong2020reasoning}.
Recent work has also incorporated graph structures into RAG. 
For example, GraphRAG organizes corpora with entity knowledge graphs, and leverages community summaries for all groups of closely related entities~\cite{edge2024local}.
LightRAG incorporates graph structures into text indexing and
relevant information retrieval for efficient RAG 
\citep{guo2024lightrag}. NodeRAG introduces heterogeneous graph structures to better leverage the structural nature of
graphs~\cite{xu2025noderag}.
However, these methods primarily use graphs to organize knowledge and guide retrieval. In contrast, we use graph structure to propagate document reliability estimates across the corpus that may contain noisy, outdated,  or contradictory information.

\noindent \textbf{Human feedback for enhancing LLMs.}
Human feedback has been used to align LLM systems with user preferences or intents. For example, InstructGPT trains language models with human feedback using reinforcement
learning~\citep{ouyang2022instructgpt}. RAG-Reward uses preference-based reward models to improve the quality of generated answers~\citep{zhang2025ragreward}. Pistis-RAG aligns human feedback by training a ranking model to improve content ranking and retrieval mechanisms~\citep{bai2024pistis}.
Some methods further study feedback in interactive settings. For example, \citet{liu2025user} leverages implicit feedback in human--LLM dialogues.
\citet{zhang2025leveraging} leverages engagement and disengagement signals for improving
LLM-based recommenders.
These works mainly use feedback to fine-tune model behavior or improve outputs in human-LLM interactions. \name instead uses limited feedback to estimate document reliability across the large-scale corpus without model fine-tuning.

\section{Methodology}
\label{sec:method}
Our objective is to enhance RAG by selecting and using more reliable documents from an external corpus for answer generation. The external document corpus may contain unreliable documents, such as those with misinformation, outdated information, or conflicting claims~\citep{xu2024knowledge, zou2024poisonedrag}. User feedback can indicate whether certain documents are reliable or unreliable. 
For example, correct and incorrect RAG-generated answers, reported by users or flagged by fact-verification systems, can be traced back to their contributing documents (e.g., using traceback methods~\cite{wang2025tracllm,cohenwang2024contextcite}), 
providing positive or negative feedback for the corresponding documents.
However, user feedback can cover only a small fraction of a large external corpus. To address this, \name{} takes advantage of document relations to propagate sparse feedback across the corpus.

As shown in Figure~\ref{fig:overview}, \name{} consists of four main stages.
First, we construct a graph where each node represents a document, and each edge encodes a relation between two documents, such as a same-source relation, mutual support, or contradiction (Section~\ref{sec:edge_labeling}).
Second, based on the document relation graph, we formulate and solve \emph{optimization problems for estimating trust scores} (Section~\ref{sec:trust_opt}). We estimate a trust score $\tau_i \in [0, 1]$ for each document $d_i$ by jointly optimizing these scores over all documents in the corpus $\mathcal{C} = \{d_1, d_2, \ldots, d_N\}$ with $N$ documents. Third, we perform \emph{trust-aware rescoring} (Section \ref{sec:reranking}).  Given a user query $q$, the objective is to produce a ranked list $\mathcal{R}_k(q)$ of $k$ documents that are provided to the LLM for answer generation, favoring query-relevant documents with higher trust scores. Finally, we perform trust-aware prompting (Section~\ref{sec:generation}) by leveraging document trust scores during answer generation.

TrustPropRAG separates offline trust estimation from query-time answer generation. Graph construction and trust score optimization are performed offline over the corpus, amortizing their computational cost across future queries, while trust-aware rescoring and prompting are performed for each query with limited additional overhead.

\begin{figure*}[t]
\centering
\includegraphics[width=2\columnwidth]{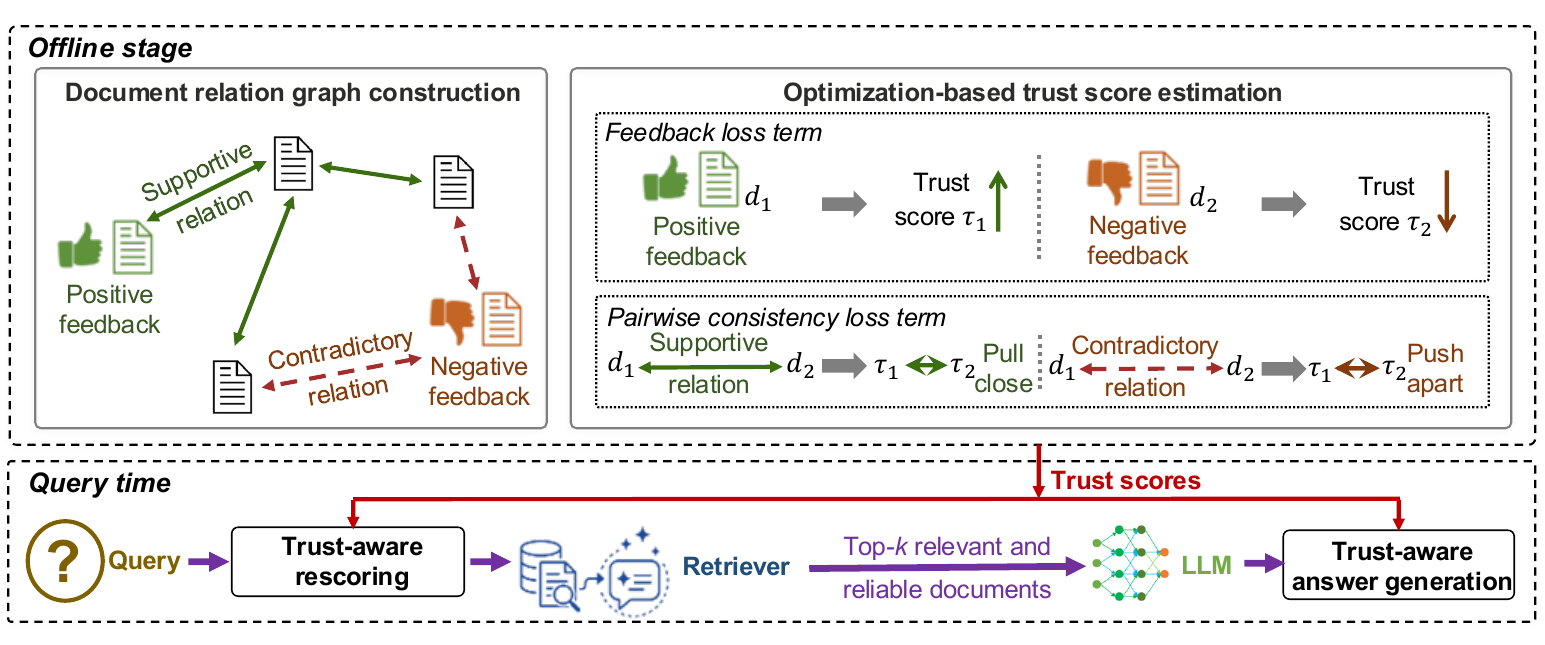}
\vspace{-0.3cm}
\caption{Overview of the \name pipeline.}
\vspace{-0.4cm}

\label{fig:overview}
\end{figure*}

\subsection{Document Relation Graph Construction}
\label{sec:edge_labeling}

Given the corpus $\mathcal{C}$, we construct a document relation graph $G=(\mathcal{C},E)$, where each node represents a document, and each edge $e_{ij}$ encodes a 
relation between documents $d_i$ and $d_j$. Each edge $e_{ij}$ is assigned 
a relation label $\ell_{ij}$. Specifically, $\ell_{ij}=+1$ means that the two documents are expected to receive similar trust scores, $\ell_{ij}=-1$ means that they contain conflicting claims and should not both be highly trusted, and $\ell_{ij}=0$ means that no reliable relation can be identified. Each edge also has a confidence weight
$w_{ij}\in[0,1]$, which reflects the confidence associated with the relation label $\ell_{ij}$.

The relation label and confidence weight can be obtained in different ways. For example, we can use source metadata to connect two documents $d_i$ and $d_j$ from the same source group with supportive edges ($\ell_{ij}=+1$),  because they are expected to have correlated trustworthiness and thus similar trust scores.
Document relations can also be inferred using an NLI model, which derives whether documents $d_i$ and $d_j$ are supportive ($\ell_{ij}=+1$), contradictory ($\ell_{ij}=-1$), or neutral ($\ell_{ij}=0$). For edges derived by NLI, the confidence weight $w_{ij}$ is set to the softmax probability associated with the relation label.

\subsection{Trust Score Optimization}
\label{sec:trust_opt}

\noindent\textbf{Pairwise consistency loss.}  
We encourage two documents with positive relation labels to have similar trust scores, while enforcing that contradictory documents have diverging trust scores. Specifically, if $d_i$ and $d_j$ are mutually supportive ($\ell_{ij} = +1$), their trust scores $\tau_i$ and $\tau_j$ should be close to each other, i.e., we aim to minimize $(\tau_i - \tau_j)^2$; whereas if $\ell_{ij} = -1$, meaning that $d_i$ and $d_j$ are contradictory, their trust scores should be far apart so that at most one document can be trusted, i.e., we aim to minimize $(\tau_i + \tau_j -1)^2$. 
Hence, we define the pairwise consistency loss $\mathcal{L}_\text{pair}$:
\begin{equation*}
\mathcal{L}_{\mathrm{pair}} = \!\!\!\sum_{\substack{e_{ij} \in E \\ \ell_{ij}=+1}} \!\!\! w_{ij} (\tau_i - \tau_j)^2 \;+\!\!\! \sum_{\substack{e_{ij} \in E \\ \ell_{ij}=-1}} \!\!\! w_{ij} (\tau_i + \tau_j - 1)^2.
\end{equation*}

\noindent\textbf{Feedback loss.}
Over time, we maintain a feedback set $\mathcal{F} = \mathcal{F}^{+} \cup \mathcal{F}^{-}$: $\mathcal{F}^{+}$ for documents marked as reliable through user feedback, and $\mathcal{F}^{-}$ for documents marked as unreliable.
We incorporate this feedback into the optimization problem of trust score estimation. Nodes in $\mathcal{F}^+$ are encouraged to have high trust scores, e.g., by minimizing $(\tau_i - \tau^+)^2$, where $\tau^+$ is close to $1$. In contrast, nodes in $\mathcal{F}^-$ are pushed toward lower trust scores, e.g., by minimizing $(\tau_i - \tau^-)^2$, where $\tau^-$ is close to $0$, to limit their influence on answer generation. We implement this by adding a feedback term $\mathcal{L}_{\mathrm{fb}}$ to the optimization problem of trust score estimation. $\mathcal{L}_{\mathrm{fb}}$ is formulated as
\begin{equation*}
\mathcal{L}_{\mathrm{fb}} = \sum_{d_i \in \mathcal{F}} (\tau_i - y_i)^2, \quad
y_i =
\begin{cases}
\tau^+, & \text{if } d_i \in \mathcal{F}^+\\
\tau^-, & \text{if } d_i \in \mathcal{F}^-
\end{cases},
\end{equation*}
where $
y_i$ is the feedback value.

\noindent\textbf{Formulating trust score optimization problem.}  To capture pairwise consistency of documents and incorporate user feedback, we jointly optimize 
objective function terms $\mathcal{L}_{\mathrm{pair}}$ and $\mathcal{L}_{\mathrm{fb}}$.  We formulate the optimization problem to assign trust scores to all documents in the corpus:
\begin{equation}
\boldsymbol{\tau}^* \;=\; \arg\min_{\boldsymbol{\tau} \in [0,1]^N}
\;\; \mathcal{L}_{\mathrm{pair}}
\;+\; \lambda_{f}\,\mathcal{L}_{\mathrm{fb}},
\label{eq:opt_problem}
\end{equation}
where $\boldsymbol{\tau}=\{\tau_1,\tau_2,\cdots,\tau_N\}$ is the set of trust scores for all documents, and $\lambda_f$ controls the trade-off between the two terms.

\noindent\textbf{Solving trust score optimization problem.} Eq.~\eqref{eq:opt_problem} is a quadratic programming (QP) problem. We solve the problem with projected gradient descent (PGD). The document relation graph can contain many nodes but remains sparse (i.e., each node has supportive or contradictory relation labels with only a small number of other nodes). As a result, each PGD step only needs to compute the loss and gradients for document pairs connected by these relations, rather than all possible document pairs. Section~\ref{sec:main_results} reports the measured cost of each stage in \name.

Before running PGD, we initialize each document’s trust score using its relations with other documents in the document relation graph. Specifically, for each document $d_i$, we compute an initial score by aggregating the signed weights of its connected edges:
$
\mathrm{init}(d_i) = \sum_{(i,j) \in E} \ell_{ij} \cdot w_{ij}
$.
Values are min-max normalized to $[\tau_\text{min}, \tau_\text{max}]$, with $0<\tau_\text{min}<0.5<\tau_\text{max}<1$. Intuitively, documents connected to others through stronger supportive relations receive higher initial trust scores, while documents involved in stronger contradictory relations receive lower initial trust scores. Graph construction and trust score optimization are performed \emph{offline} over the corpus, amortizing their computational cost across future queries. 

\noindent\textbf{Theoretical analysis.} 
Trust scores exhibit \emph{multi-hop propagation}. Feedback on a document shifts its trust score; through pairwise document relations, this effect then spreads to its neighbors, which further propagate the influence to their neighbors.
We formalize this behavior in two lemmas below. 
Since $\mathcal{L}=\mathcal{L}_{\mathrm{pair}}
\;+\; \lambda_{f}\,\mathcal{L}_{\mathrm{fb}}$ is quadratic in $\boldsymbol{\tau}$, its Hessian
$A = \nabla^2 \mathcal{L}$ is symmetric positive semidefinite.
We assume that every connected component of $G$ contains at least one document with feedback, which is sufficient
to make $A$ strictly positive definite (see proof in Appendix~\ref{app:proofs}). Let $\sigma_{\min} \leq \cdots \leq \sigma_{\max}$ denote its eigenvalues
and $\kappa = \sigma_{\max}/\sigma_{\min}$ its condition number.
The unconstrained minimizer satisfies the linear system
$A\boldsymbol{\tau}^* = \mathbf{b}$. We use this unconstrained solution to analyze trust propagation and convergence rate.

\begin{lemma}[Trust propagation]
\label{thm:localization} Let $f_1,\ldots,f_m$ be the feedback documents in $\mathcal{F}$. 
Suppose the feedback value $y_k$ at each $f_k \in \mathcal{F}$ is
perturbed with a change of $\Delta y_{f_k}$.
Then the resulting change in the trust score of document $d_i$ satisfies
\begin{equation*}
\label{eq:superposition}
\lvert\Delta\tau_i\rvert \leq \sum_{k=1}^{m} \gamma_k \cdot q^{\,d_G(i,\, f_k)},
\qquad q = \frac{\sqrt{\kappa}-1}{\sqrt{\kappa}+1},
\end{equation*}
where $\gamma_k = 2\lambda_f \gamma |\Delta y_{f_k}|$, $\gamma=(1+\sqrt{\kappa})^2/(2\sigma_{\max})=O(1/\sigma_{\min})$, and $d_G(i,j)$ is the shortest-path distance between nodes $i$ and $j$ in $G$. \end{lemma}
\emph{Proof.} See Appendix~\ref{proof1}.

This lemma shows that trust propagation is distance-weighted and cumulative.
Each document receives influence from all documents with feedback. Each influence
term is weighted by $q^{d_G(i,f_k)}$, which depends on its graph distance to the corresponding document with feedback.
The result suggests that sparse feedback can still be effective when the documents with feedback are well
positioned in the relation graph, allowing their influence to reach other
documents through short relation paths.

\begin{lemma}[Convergence]
\label{thm:convergence}
With the fixed step size
$\eta^\star = 2/(\sigma_{\min}+\sigma_{\max})$, the projected gradient descent
iterates satisfy
\begin{equation*}
\label{eq:convergence}
\lVert\boldsymbol{\tau}^{(t)} - \boldsymbol{\tau}^*\rVert
\;\leq\; \rho \cdot \lVert\boldsymbol{\tau}^{(t-1)} - \boldsymbol{\tau}^*\rVert,
\qquad \rho = \frac{\kappa-1}{\kappa+1}.
\end{equation*}
\end{lemma}
\emph{Proof.}
See Appendix~\ref{proof2}.

At each iteration, the error contracts by a factor of at most $\rho$, so the number of iterations needed to reach an $\epsilon$-accurate
solution (i.e.\ 
$\lVert\boldsymbol{\tau}^{(t)} - \boldsymbol{\tau}^*\rVert \leq \epsilon$) scales as
$O(\kappa\log(1/\epsilon))$. In our experiments, PGD converges
within $1000$ iterations on tested datasets.

\subsection{Trust-aware Rescoring}
\label{sec:reranking}
We select documents with both high trust scores and high query relevance, and provide them to the LLM as retrieved context for answer generation. To this end, we perform trust-aware rescoring for every document by using its estimated trust score and its similarity score. The similarity score is computed by the retriever, e.g., as a dot product score, to measure how relevant the document is to the query. Specifically, for document $d_i$, we combine its trust score $\tau_i$ with its normalized similarity score ${s}_i$ and define the trust-aware ranking score as $
r_i = \alpha \tau_i + (1-\alpha){s}_i$,
where $\alpha \in [0,1]$ controls the trade-off between document trustworthiness 
and its relevance to the query. We select the top-$k$ documents with the largest $r_i$ and provide them to the LLM for answer generation.

\subsection{Trust-aware Answer Generation}
\label{sec:generation}

Apart from using trust-aware rescoring for document retrieval, we further leverage the estimated trust scores during answer generation. Specifically, for the top-$k$ documents with the largest trust-aware ranking scores, we format the prompt so that each document is accompanied by its trust score. We also include an instruction in the prompt to guide the LLM to place greater reliance on documents with higher trust scores during answer generation. This trust-aware prompting complements trust-aware rescoring: while rescoring filters out less reliable documents before answer generation, this trust-aware prompting helps the LLM account for document reliability during answer generation.

\section{Evaluation}
\label{sec:eval}

\subsection{Evaluation Setup} 
\label{sec:impl}

\noindent\textbf{Datasets.} We evaluate \name  with three open-domain QA datasets: MS~MARCO~\citep{bajaj2016msmarco}, Natural Questions (NQ)~\citep{kwiatkowski2019nq}, and TriviaQA~\citep{joshi2017triviaqa}. These datasets are widely adopted in retrieval-augmented and open-domain QA evaluation, and together they cover diverse question types. Following the evaluation settings commonly used in prior RAG reliability and robustness studies \citep{shen2025reliabilityrag, zhong2023poisoning,wei2024instructrag,wang2025astuterag}, we randomly sample 300 query--answer instances for each dataset. Half of the sampled queries are paired with synthetic contradictory documents generated by following the method in PoisonedRAG~\citep{zou2024poisonedrag}. These documents are written to remain fluent and relevant to the query while supporting answers that conflict with the ground-truth answers. The remaining queries are paired only with factual documents from the original datasets. In this way, we construct the evaluation corpus from the relevant documents associated with the sampled queries, resulting in 805, 893, and 1,037 documents for MS MARCO, NQ, and TriviaQA, respectively. This setup evaluates whether the end-to-end system can produce correct answers when the corpus contains both reliable evidence and fluent but misleading contradictory content.

\noindent\textbf{Retrievers.}
We evaluate with both sparse and dense retrievers. We use BM25~\citep{robertson2009bm25}, a sparse lexical retriever, and two dense bi-encoder retrievers: Contriever~\citep{izacard2022contriever} and MiniLM based on \texttt{all-MiniLM-L6-v2}~\citep{reimers2019sentence}. For each query, we retrieve $k{=}5$ documents.

\noindent\textbf{Graph construction.} We construct a document relation graph using two types of information: document sources and semantic relations between documents.
Documents $i$ and $j$ from the same source are connected with supportive edges ($\ell_{ij}=+1$). In TriviaQA, we determine whether two documents share the same source based on their source Uniform Resource Locators (URLs). For MS MARCO and NQ, which do not provide source metadata, we use simulated source grouping: factual documents are partitioned into multiple source groups, and contradictory documents are partitioned into another set of source groups. Documents in the same group are treated as sharing the same source. We adopt this simulation as a proxy for real source metadata, following recent source-reliability-aware RAG work that simulates sources with varying reliability~\citep{hwang2025retrieval}. We also apply an NLI model (\texttt{nli-deberta-v3-small}) to each document pair. If the model predicts that two documents $i$ and $j$ support each other, we add a supportive edge ($\ell_{ij}=+1$). If it predicts that they contradict each other,
we add a contradictory edge ($\ell_{ij}=-1$). Appendix~\ref{app:nli_quality} reports the precision and recall of the
NLI-derived edges on all three datasets. Table~\ref{tab:fb_baseline} evaluates the impact of removing the source-group edges to isolate the contribution of the NLI-derived edges.

\noindent\textbf{Feedback construction.} We construct the feedback set by randomly sampling a fraction of factual and contradictory documents and simulating feedback labels for them.
Sampled factual documents are assigned positive feedback labels, while sampled contradictory documents are assigned negative feedback labels. We refer to this sampling fraction as the feedback ratio (FB); for instance, FB = 30\% means that feedback labels are provided for 30\% of the documents. As real-world feedback may be imperfect, we evaluate \name under noisy feedback in Figure~\ref{fig:feedback_ratio}, where a portion of the feedback labels are incorrect.

\noindent\textbf{Trust optimization.}
For PGD, we use a step size of $\eta = 0.01$, run for at most $T = 1000$ iterations, and stop when the convergence tolerance reaches $\epsilon = 10^{-6}$.
We set feedback loss weight $\lambda_f{=}1.0$.
Trust scores are optimized over all corpus documents, resulting in large-scale optimization problems over 805, 893, and 1,037 documents for MS MARCO, NQ, and TriviaQA, respectively.

\noindent\textbf{Rescoring.}
We set the trust coefficient to $\alpha{=}0.6$. We also test the impact of $\alpha$ in Section~\ref{sec:main_results}. The top-$k$ documents, where $k=5$ in our experiments, are provided as retrieved context to the LLM. We also test the impact of $k$ in Appendix~\ref{app:k_ablation}.

\noindent\textbf{Answer generation.} 
We use proprietary and open-source LLMs, including GPT-4o-mini, GPT-5, Gemini~2.5~Flash, Gemini 2.5 Pro, Llama-3.1-8B-Instruct, and Mistral-7B-Instruct. \name leverages the trust-aware prompt that annotates each document with its trust score (as detailed in Section~\ref{sec:generation}).

\noindent\textbf{Evaluation metrics.} We adopt exact match (EM) and fact precision@$k$ (FP@$k$) as evaluation metrics.
EM is a widely used metric
in open-domain QA.
EM checks whether the normalized generated answer contains the ground-truth answer, with normalization including lowercasing and the removal of articles and punctuation. FP@$k$ evaluates retrieval quality, which measures the reliability of the top-ranked documents used for answer generation. Specifically, let $\mathcal{D}_{\mathrm{fact}}$ and $\mathcal{D}_{\mathrm{contr}}$ denote the sets of factual and contradictory documents, respectively, and let $\mathcal{R}_k(q)$ denote the top-$k$ documents for query $q$. 
FP@$k$ is defined as the fraction of factual documents among these top-$k$ documents:
$
\label{eq:fp}
\text{FP@}k =
\frac{
|\mathcal{R}_k(q) \cap \mathcal{D}_{\mathrm{fact}}|
}{
|\mathcal{R}_k(q) \cap (\mathcal{D}_{\mathrm{fact}} \cup \mathcal{D}_{\mathrm{contr}})|
}.
$

\noindent\textbf{Baselines.}
We compare \name against the following methods:
(1)~\textit{VanillaRAG} ranks documents using the relevance scores produced by the retriever and feeds the top-ranked documents to the LLM. 
(2)~\textit{InstructRAG}~\citep{wei2024instructrag} guides the LLM to better use retrieved evidence during generation. 
It decomposes retrieved documents into atomic claims, and prompts the LLM to generate a rationale connecting relevant claims to the final answer. (3)~\textit{AstuteRAG}~\citep{wang2025astuterag} combines the LLM's internal knowledge with retrieved evidence. 
It first prompts the LLM to produce an answer-relevant document from its internal 
knowledge, and then compares this document with the retrieved documents and resolves conflicts before generating the answer. 
(4)~\textit{ReliabilityRAG}~\citep{shen2025reliabilityrag} uses NLI to identify contradictions among retrieved documents. It then solves a maximum independent set problem to select a subset of documents with no detected pairwise contradictions and uses keyword aggregation to generate the final answer.
(5)~\textit{TrustRAG}~\citep{zhou2025trustrag} includes a first-stage filtering procedure that clusters the retrieved documents by embedding similarity and filters out clusters identified as suspicious. We use this first-stage procedure as the baseline.

\subsection{Evaluation Results}
\label{sec:main_results}

\begin{table}[t]
\centering
\small
\setlength{\tabcolsep}{3pt}
\begin{tabular}{@{}llccc@{}}
\toprule
Dataset & Method
  & Contriever & BM25 & MiniLM \\
\midrule
\multirow{5}{*}{
{\small MS MARCO}}
& VanillaRAG        & 0.27 & 0.24 & 0.26 \\
& InstructRAG       & 0.35 & \underline{0.40} & 0.37 \\
& AstuteRAG         & 0.38 & 0.39 & 0.37 \\
& ReliabilityRAG    & \underline{0.42} & \underline{0.40} & \underline{0.39} \\
& TrustRAG\textsubscript{stage 1}   & 0.36 & 0.35 & 0.33 \\
\cmidrule(l){2-5}
& \name             & \textbf{0.45} & \textbf{0.47} & \textbf{0.44} \\
\midrule
\multirow{5}{*}{
{\small NQ}}
& VanillaRAG        & 0.37 & 0.37 & 0.39 \\
& InstructRAG       & \underline{0.46} & 0.48 & 0.43 \\
& AstuteRAG         & 0.44 & \underline{0.50} & 0.42 \\
& ReliabilityRAG    & 0.45 & \underline{0.50} & \underline{0.48} \\
& TrustRAG\textsubscript{stage 1}   & 0.44 & 0.46 & 0.41 \\
\cmidrule(l){2-5}
& \name             & \textbf{0.55} & \textbf{0.59} & \textbf{0.59} \\
\midrule
\multirow{5}{*}{
{\small TriviaQA}}
& VanillaRAG          & 0.57 & 0.62 & 0.49 \\
& InstructRAG       & 0.68 & 0.67 & 0.69 \\
& AstuteRAG         & 0.70 & 0.71 & \underline{0.72} \\
& ReliabilityRAG    & \underline{0.75} & \underline{0.72} & 0.67 \\
& TrustRAG\textsubscript{stage 1}   & 0.66 & 0.65 & 0.63 \\
\cmidrule(l){2-5}
& \name             & \textbf{0.78} & \textbf{0.75} & \textbf{0.77} \\
\bottomrule
\end{tabular}
\caption{EM averaged over GPT-4o-mini, Gemini 2.5 Flash, and GPT-5
($k{=}5$, 
FB=30\%).
Best results are \textbf{bolded}; second-best results are \underline{underlined}.}
\label{tab:main_em}
\end{table}

\begin{table}[t]
\centering
\small
\setlength{\tabcolsep}{4pt}
\begin{tabular}{@{}llcc@{}}
\toprule
Dataset & Retriever
  & Baselines & \name \\
\midrule
\multirow{3}{*}{MS~MARCO}
& Contriever & 44\% & \textbf{77\%} \\
& BM25       & 57\% & \textbf{92\%} \\
& MiniLM     & 54\% & \textbf{88\%} \\
\midrule
\multirow{3}{*}{NQ}
& Contriever & 53\% & \textbf{79\%} \\
& BM25       & 46\% & \textbf{88\%} \\
& MiniLM     & 48\% & \textbf{96\%} \\
\midrule
\multirow{3}{*}{TriviaQA}
& Contriever & 38\% & \textbf{77\%} \\
& BM25       & 57\% & \textbf{90\%} \\
& MiniLM     & 49\% & \textbf{92\%} \\
\bottomrule
\end{tabular}
\caption{Retrieval quality measured by FP@5 ($k{=}5$, FB=30\%). }

\label{tab:fp5}
\end{table}

\noindent\textbf{\name outperforms baselines on EM and FP@5.} Tables~\ref{tab:main_em} and~\ref{tab:fp5} 
report the EM and FP@5 averaged over GPT-4o-mini, Gemini~2.5~Flash, and GPT-5. 
Table~\ref{tab:main_em} shows that \name achieves the best EM across all datasets and retrievers, with absolute EM gains of 0.03–0.11 compared with the strongest baseline.
Table~\ref{tab:fp5} shows FP@5 (the fraction of factual documents among top-$k$ documents) of \name and baselines. All baselines have the same FP@5 as VanillaRAG under this retrieval-level metric: InstructRAG, AstuteRAG, and TrustRAG$_{\text{stage 1}}$ operate after the retriever selects the top-\(k\) documents based on query relevance and thus do not change the ranked top-$k$ documents; 
although ReliabilityRAG filters documents before final answer generation, its filtering is driven by intermediate LLM-generated answers and their consistency, rather than by producing a new 
top-$k$ list. The baseline methods have an FP@5 of 38–57\%, while TrustPropRAG raises it to 77–96\%. The improvement is observed with all three retrievers, indicating that the method is not tied to a specific retriever.
These improvements in FP@5 help explain the EM gains of \name. Specifically, a higher FP@5 indicates that the top-$k$ list contains a larger fraction of factual documents, which suggests that trust score optimization assigns higher trust to more reliable documents, and trust-aware rescoring promotes them into the top-$k$ list. As a result, the LLM inputs contain more documents that support the correct answer.

\noindent\textbf{Ablation study.}
We conduct an ablation study to examine the contribution of each component in \name. Table~\ref{tab:component} evaluates five variants on MS~MARCO with FB=30\% using Gemini~2.5~Flash. We observe that the combination of trust score optimization and trust-aware rescoring leads to a substantial improvement. Compared with VanillaRAG, adding these two components increases FP@5 from 54\% to 87\% and improves EM from 0.28 to 0.41.
This suggests that the optimized trust scores are effective for identifying more trustworthy documents and that rescoring can promote them into the top-$k$ context used for answer generation. We also observe that trust-aware prompting provides an additional benefit when combined with trust score optimization and trust-aware rescoring, increasing EM from 0.41 to 0.46.
This suggests that some low-trust documents may still remain in the top-$k$ list after rescoring, and trust-aware prompting helps the LLM reduce their influence during answer generation. 
Finally, we observe that using trust-aware prompting alone does not improve performance. Without trust score optimization, FP@5 remains at 54\%, and EM slightly decreases from 0.28 to 0.27 compared with VanillaRAG. This suggests the trust-aware prompt becomes useful only when the trust scores have been optimized and can provide a meaningful signal for answer generation.

\begin{table}[t]
\centering
\small
\setlength{\tabcolsep}{2pt}
\begin{tabular}{@{}lccc| cc@{}}
\toprule
Variant &Opt. & Rescore & Prompt & EM & FP@5 \\
\midrule
VanillaRAG & -- & -- & -- & 0.28 & 54\% \\
Opt.+Prompt & \checkmark & -- & \checkmark & 0.35 & 54\% \\
Opt.+Rescore & \checkmark & \checkmark & -- & 0.41 & 87\% \\
Prompt only & -- & -- & \checkmark & 0.27 & 54\% \\
Full \name & \checkmark & \checkmark & \checkmark & 0.46 & 88\% \\
\bottomrule
\end{tabular}
\caption{Effects of different TrustPropRAG components.
\emph{Opt}: trust score optimization, \emph{Rescore}: trust-aware rescoring, and \emph{Prompt}: trust-aware prompting.}
\vspace{-0.3cm}
\label{tab:component}
\end{table}

\begin{table}[t]
\centering
\small
\setlength{\tabcolsep}{4pt}
\begin{tabular}{@{}lcc@{}}
\toprule
Variant & EM & FP@5 \\
\midrule
VanillaRAG & 0.31 & 51\% \\
\name (trivial feedback) & 0.36 & 62\% \\
\name (NLI-only) & 0.44 & 79\% \\
\name & 0.49 & 92\% \\
\bottomrule
\end{tabular}
\caption{Effects of trust propagation and source-group edges (MS~MARCO and NQ, MiniLM, GPT-4o-mini, $k{=}5$, FB=30\%).}
\label{tab:fb_baseline}
\end{table}

\begin{table*}[t]
\centering
\small
\setlength{\tabcolsep}{3pt}
\begin{tabular}{@{}l cc cc cc c@{}}
\toprule 
\multirow{3}{*}{
{\small Model}} & \multicolumn{2}{c}{MS~MARCO}
& \multicolumn{2}{c}{NQ}
& \multicolumn{2}{c}{TriviaQA}
& \multirow{3}{*}{
{\small $\overline{\Delta}$}}\\
\cmidrule(lr){2-3} \cmidrule(lr){4-5} \cmidrule(lr){6-7}
 & VanillaRAG & \name & VanillaRAG & \name & VanillaRAG & \name &  \\
\midrule
\multicolumn{8}{@{}l}{\emph{Proprietary}} \\
GPT-4o-mini       & 0.21 & 0.40 & 0.41 & 0.58 & 0.50 & 0.76 & +0.21 \\
GPT-5             & 0.30 & 0.49 & 0.43 & 0.62 & 0.53 & 0.85 & +0.23 \\
Gemini~2.5~Flash  & 0.28 & 0.46 & 0.33 & 0.56 & 0.44 & 0.69 & +0.22 \\
Gemini~2.5~Pro    & 0.32 & 0.48 & 0.37 & 0.60 & 0.46 & 0.77 & +0.23 \\
\midrule
\multicolumn{8}{@{}l}{\emph{Open-source}} \\
Llama-3.1-8B-Instruct & 0.19 & 0.32 & 0.38 & 0.55 & 0.48 & 0.77 & +0.20 \\
Mistral-7B-Instruct   & 0.22 & 0.35 & 0.33 & 0.47 & 0.51 & 0.74 & +0.17 \\
\bottomrule
\end{tabular}
\caption{Performance with different LLMs across three datasets.
$\overline{\Delta}$ denotes \name's absolute EM improvement over VanillaRAG, averaged across the three datasets.}
\label{tab:model}
\end{table*}

\begin{figure}[t]
  \centering
  \includegraphics[width=0.72\columnwidth]{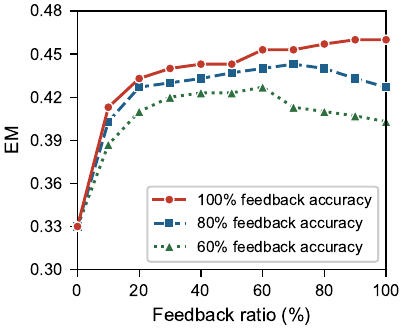}
  \vspace{-0.3cm}
  \caption{Impact of feedback ratio and feedback accuracy on EM (MS~MARCO, MiniLM, Gemini~2.5~Flash).}
  \label{fig:feedback_ratio}
\end{figure}

\noindent\textbf{The effects of trust propagation and source-group edges.} To separate the benefit of trust propagation from simply having access to feedback, the \emph{trivial feedback} variant uses the same feedback set as \name but does not propagate the feedback through the graph. Documents with positive feedback receive a boost to their retrieval scores, while documents with negative feedback are not included in the top-$k$ list. To isolate the contribution of the NLI-derived edges from source-group edges, the \emph{NLI-only} variant removes all simulated source-group edges and retains only the NLI-derived edges.

From Table~\ref{tab:fb_baseline}, the trivial feedback variant achieves an EM of 0.36 and an FP@5 of 62\%. This confirms that directly using feedback is beneficial. However, its performance remains substantially below TrustPropRAG, which achieves an EM of 0.49 and an FP@5 of 92\%, respectively. 
 This gap shows that the gain of \name cannot be attributed only to access to the feedback set. Propagating feedback through document relation graphs provides trust estimates for documents without direct feedback.  Table~\ref{tab:fb_baseline} also shows that the
NLI-only variant achieves an EM of 0.44 and an FP@5 of 79\%. Although removing the source-group edges leads to a performance drop, the NLI-only variant still outperforms the trivial feedback variant, showing that propagation remains beneficial even when the graph contains only NLI-derived relations.

\noindent\textbf{Performance under noisy feedback.}
Figure~\ref{fig:feedback_ratio} evaluates the robustness of TrustPropRAG by varying the accuracy of feedback labels from 60\% to 100\%.
Feedback accuracy denotes the fraction of feedback labels that are correct. For example, 80\% feedback accuracy means that  20\% of the feedback labels are flipped: a factual document may be labeled as low-trust, or a contradictory document may be labeled as high-trust.
Figure~\ref{fig:feedback_ratio} reports EM on MS~MARCO with Gemini~2.5~Flash. We observe that \name is robust to noisy feedback.
Even when feedback accuracy is only 60\% (i.e., 40\% of feedback labels are incorrect),
\name still improves EM over the setting without feedback across all feedback ratios. This is because individual feedback errors can be mitigated through trust propagation over the document relation graph, where relations among documents help mitigate the impact of noisy feedback signals.

\begin{figure}[t]

  \centering
  \includegraphics[width=0.82\columnwidth]{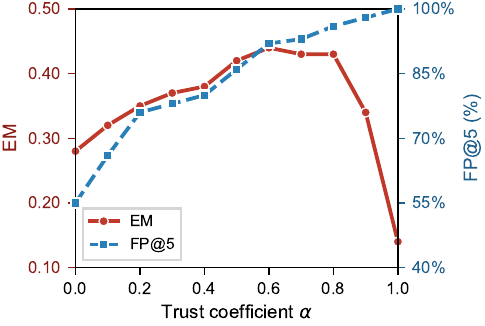}
              \vspace{-0.3cm}
  \caption{Impact of trust coefficient $\alpha$ on EM and FP@5.
           (MS~MARCO, MiniLM, Gemini~2.5~Flash, FB=30\%)}
  \label{fig:alpha}

\end{figure}

\noindent\textbf{The impact of feedback ratio.}
Figure~\ref{fig:feedback_ratio} also shows EM as a function of the feedback ratio on MS~MARCO with Gemini~2.5~Flash.
Compared with the setting without feedback, \name improves sharply when only 10\% of documents receive feedback, after which the improvement becomes more gradual. This indicates that a small portion of feedback is sufficient for trust propagation to produce reliable trust estimates across the document relation graph.

\noindent\textbf{The impact of trust coefficient $\alpha$.}
Figure~\ref{fig:alpha} shows EM and FP@5 as the trust coefficient
$\alpha$ varies from 0 to 1. When $\alpha=0$, rescoring uses only the retrieval relevance score.
As $\alpha$ increases from $0$, both EM and FP@5 improve, showing that incorporating trust scores helps promote documents that support the correct answer into the top-$k$ context.
EM reaches its best value at $\alpha=0.6$, and remains stable between $0.5$ and $0.8$.
However, when $\alpha$ is too large, the top-$k$ list tends to include documents with high trust scores but lower query relevance, leading to a slight decrease in EM.
As \name is not sensitive to the exact choice of $\alpha$ within a range (e.g., between 0.5 and 0.8), we set $\alpha{=}0.6$ for the rest of the experiments.

\noindent\textbf{Performance with different LLMs.}
We evaluate the performance of \name with a diverse set of 
proprietary and open-source 
LLMs of different model scales in Table~\ref{tab:model}.
Since FP@5 depends only on retrieval and rescoring, it is unchanged across LLMs; therefore, we report EM as the main metric.
We use MiniLM as the retriever and set $k=5$ and FB=30\%. Table~\ref{tab:model} shows that \name consistently outperforms VanillaRAG across all evaluated LLMs. The average EM improvement ranges from 0.17 to 0.23 across both proprietary models (GPT-4o-mini, GPT-5, Gemini~2.5~Flash, Gemini~2.5~Pro) and smaller open-source models (Llama-3.1-8B-Instruct, Mistral-7B-Instruct).
Notably, \name still achieves an average EM gain of 0.23 with strong LLMs such as GPT-5 and Gemini~2.5~Pro. This suggests that improving the reliability of the retrieved context with trust propagation remains useful even for strong LLMs.

\noindent\textbf{Computational cost.}
Table~\ref{tab:cost} reports the computational cost of TrustPropRAG, including offline graph construction and trust optimization, as well as query-time retrieval. The computational resources are detailed in Appendix~\ref{resource}.
Constructing the graph accounts for most of the offline cost, taking 578.3--1077.1 seconds. Given the constructed graph, trust score optimization is efficient, requiring only 0.3--0.4 seconds, indicating that trust propagation itself introduces negligible computational overhead. 
Both graph construction and trust optimization are performed offline and their costs are therefore amortized across future queries. 
At query time, retrieval takes 4.4--4.8 seconds and is shared by TrustPropRAG and the baselines. 
Beyond retrieval, TrustPropRAG only performs lightweight trust-aware rescoring using the precomputed trust scores, introducing limited additional query-time overhead.

\begin{table}[t]
\centering
\small
\setlength{\tabcolsep}{3pt}
\begin{tabular}{@{}lccc@{}}
\toprule
Stage & MS~MARCO & NQ & TriviaQA \\
\midrule
\multicolumn{4}{@{}l}{\emph{Offline}} \\
Graph construction & 578.3 s & 852.4 s & 1077.1 s \\
Trust optimization (PGD) & 0.3 s & 0.3 s & 0.4 s \\
\midrule
\multicolumn{4}{@{}l}{\emph{Query time}} \\
Retrieval & 4.4 s & 4.8 s & 4.8 s \\
\bottomrule
\end{tabular}
\caption{Computational cost breakdown.}
\label{tab:cost}
\end{table}

\section{Conclusions}
\label{sec:conclusions}

We present TrustPropRAG, a framework that improves RAG reliability by propagating sparse human feedback over a document relation graph. TrustPropRAG estimates a trust score for each document by optimizing an objective that combines pairwise consistency loss and feedback loss. The resulting trust scores are used to select more reliable documents and guide trust-aware answer generation. TrustPropRAG achieves absolute EM gains of 0.03–0.11 over the strongest baseline, and demonstrates robustness to noisy feedback.

\section*{Acknowledgments}
We thank the anonymous reviewers 
for insightful feedback. This work was supported by Seed Grant of IST, and the National Science Foundation under Grants 2550742, 2623125, 2555329, and 2549266.

\section*{Limitations}

\name incorporates human feedback as binary labels, marking documents as either reliable or unreliable. In practice, feedback can be graded rather than binary. Studying \name under graded feedback is left to future work. Following common practice in RAG evaluation, we focus on unreliable content that is injected as synthetic contradictions rather than drawn from naturally occurring conflicts. 
Finally, we focus on open-domain QA with short factual answers. Applying trust propagation to long-form or multi-hop generation, where reliability interacts with compositional reasoning, is left to future work.

\section*{Ethical Considerations}

This work addresses the problem of unreliable content in retrieval-augmented generation. Our experiments use synthetically generated contradictions.
All datasets and models used in this work are publicly released and licensed for research use, and our usage is consistent with their intended research purposes. MS MARCO, NQ, and TriviaQA are publicly available open-domain QA benchmarks released for research. The retrievers (BM25, Contriever, and the all-MiniLM-L6-v2 encoder) and the NLI model (nli-deberta-v3-small) are also publicly released.

\bibliography{ref}

\appendix
\section{Implementation Details}
\label{resource}
\noindent \textbf{Models and computational resources}. Among the LLMs we evaluate, Llama-3.1-8B-Instruct and Mistral-7B-Instruct are open-source models with 8B and 7B parameters, respectively; GPT-4o-mini, GPT-5, Gemini 2.5 Flash, and Gemini 2.5 Pro are proprietary models accessed through their respective APIs. The NLI model (nli-deberta-v3-small) has approximately 44M parameters, and the dense retrievers (Contriever and all-MiniLM-L6-v2) are sentence-encoder models with approximately 110M and 22M parameters respectively. All local computation, including NLI-based edge construction, dense retrieval, and trust score optimization, was run on an NVIDIA RTX 5070 GPU. For our largest corpus (1,037 documents, TriviaQA), this took approximately 0.3 GPU-hours. Trust score optimization via projected gradient descent is lightweight, converging within 1,000 iterations and taking under 1 minute per corpus. We implement retrieval, NLI inference, and trust score optimization using standard Python libraries, and provide exact package versions and scripts in the released code.

\noindent \textbf{Hyperparameters.}
We use the same hyperparameters across all datasets, retrievers, and LLMs. For trust score initialization, we set $\tau_{\min}=0.3$ and $\tau_{\max}=0.7$.
We set feedback loss weight $\lambda_f{=}1.0$. We set the
feedback values $\tau^{+}=0.9$ and $\tau^{-}=0.1$.

\section{Number of Retrieved Documents $k$}
\label{app:k_ablation}

The number of retrieved documents $k$ controls the size of the context provided to the LLM. A smaller $k$ yields a more selective context but may omit relevant evidence, while a larger $k$ provides broader coverage at the cost of allowing more unreliable documents into the prompt. We evaluate TrustPropRAG across $k \in \{1, 3, 5, 10, 20\}$ on MS MARCO with MiniLM as the retriever, Gemini 2.5 Flash as the backbone LLM, and FB$=$30\%. All other hyperparameters follow Section~\ref{sec:eval}.

\begin{table}[t]
\centering
\small
\begin{tabular}{lccccc}
\toprule
& $k$=1 & $k$=3 & $k$=5 & $k$=10 & $k$=20 \\
\midrule
EM & 0.41 & 0.43 & 0.46 & 0.45 & 0.42 \\
FP@$k$ & 95\% & 92\% & 88\% & 84\% & 73\% \\
\bottomrule
\end{tabular}
\caption{Impact of the number of retrieved documents $k$ on EM and FP@$k$ (MS MARCO, MiniLM, Gemini 2.5 Flash, FB$=$30\%).}
\label{tab:k_ablation}
\end{table}

FP@$k$ falls steadily from 95\% ($k$=1) to 73\% ($k$=20) as lower-trust documents enter the top-$k$ selected documents. EM is less sensitive, staying within 0.41--0.46 and peaking at $k$=5 and $k$=10. Trust-aware rescoring places reliable evidence at the top, so small $k$ already provides sufficient factual support; at larger $k$, the trust-aware prompt helps the LLM put less weight to lower-trust documents that are provided in the context. EM is smaller at $k$=1, where the retrieved context may provide insufficient context, and at $k$=20, where the larger retrieval set exposes the LLM to more contradictory context.

Considering that EM peaks at $k$=5 and $k$=10, we adopt $k$=5, matching prior work~\citep{zou2024poisonedrag, shen2025reliabilityrag}.

\section{Proofs}
\label{app:proofs}

We provide proofs of the two lemmas stated in Section~\ref{sec:trust_opt}.
Both lemmas rely on the fact that the trust optimization objective
$\mathcal{L} = \mathcal{L}_{\mathrm{pair}} + \lambda_f \mathcal{L}_{\mathrm{fb}}$
is a quadratic in $\boldsymbol{\tau}$, so the first-order optimality condition for the unconstrained quadratic problem is given by the linear system $A\boldsymbol{\tau}^* = \mathbf{b},$
where $A = \nabla^2 \mathcal{L}$ is the Hessian.
We first verify that $A$ is positive definite under the assumption that every connected component of $G$ contains at least one document with feedback, then prove the two lemmas.

\paragraph{Positive definiteness of $A$.} Let $\mathbf{e}_i \in \mathbb{R}^N$ denote the $i$-th standard basis vector, whose $i$-th entry is one and all other entries are zero.
Expanding the Hessian gives
\begin{align}
\label{eq:A}
A &= 2\!\!\sum_{\substack{e_{ij}\in E\\ \ell_{ij}=+1}}\!\! w_{ij}\,(\mathbf{e}_i - \mathbf{e}_j)(\mathbf{e}_i - \mathbf{e}_j)^{\top} \nonumber\\
  &\quad + 2\!\!\sum_{\substack{e_{ij}\in E\\ \ell_{ij}=-1}}\!\! w_{ij}\,(\mathbf{e}_i + \mathbf{e}_j)(\mathbf{e}_i + \mathbf{e}_j)^{\top} \nonumber\\
  &\quad + 2\lambda_f\!\sum_{i\in \mathcal{F}} \mathbf{e}_i \mathbf{e}_i^{\top},
\end{align}
which is a sum of positive semidefinite matrices, hence $A \succeq 0$. Next, we prove $A \succ 0$.
A vector $\mathbf{v}$ lies in $\ker(A)$ iff
(a) $v_i = v_j$ on every supportive edge ($\ell_{ij}=+1$),
(b) $v_i = -v_j$ on every contradictory edge ($\ell_{ij}=-1$), and
(c) $v_i = 0$ on every node with feedback.
If a node with value zero is connected to
another node by a supportive edge, condition (a) implies that the neighboring
node also has value zero. If it is connected by a contradictory edge, condition
(b) again implies that the neighboring node has value zero.
Under our assumption that every connected component contains at least one
feedback node, we have $\mathbf{v}=\mathbf{0}$ globally, so $A \succ 0$.
We write its eigenvalues as $\sigma_{\min} \le \cdots \le \sigma_{\max}$ and its condition number as
$\kappa = \sigma_{\max}/\sigma_{\min}$.

\subsection{Proof of Lemma~\ref{thm:localization} (Trust Propagation)}
\label{proof1}

We fix the set of documents with feedback $\mathcal{F} = \{f_1,\ldots,f_m\}$ and consider
perturbations of their feedback values. Let $\Delta y_{f_k}$ denote the perturbation applied to the feedback value
$y_{f_k}$. These perturbations leave $A$ unchanged and only modify
the right-hand side $\mathbf{b}$ at entries corresponding to documents in
$\mathcal{F}$. Let $\Delta\mathbf{b}$ denote the resulting
change in $\mathbf{b}$, and let $\Delta b_{f_k}$ denote its entry corresponding
to feedback document $f_k$. We have
 \begin{equation}
 \label{eq:b_f_k}
 \Delta b_{f_k} = 2 \lambda_f\Delta y_{f_k}.
 \end{equation}

\paragraph{Step 1: Linearity and superposition.} Because the perturbations only change the feedback target values, the matrix
$A$ remains fixed. Therefore, the unconstrained optimum satisfies
$
\boldsymbol{\tau}^* = A^{-1}\mathbf{b}.
$
After the perturbation, the right-hand side becomes
$\mathbf{b}+\Delta\mathbf{b}$, and the corresponding change in the optimized
trust scores is
\[
\Delta\boldsymbol{\tau}^{*}
=
A^{-1}\Delta\mathbf{b}.
\]
Since $\Delta\mathbf{b}$ is nonzero only at documents with feedback, the change in
the trust score of document $d_i$ can be written as
\begin{equation}
\label{eq:delta_t_i}
\Delta\tau_i
=
\sum_{k=1}^{m}(A^{-1})_{i f_k}\Delta b_{f_k},
\end{equation}
where $(A^{-1})_{if_k}$ denotes the entry in the $i$-th row and $f_k$-th column of
$A^{-1}$.

\paragraph{Step 2: Entry-wise decay of $A^{-1}$.}
The feedback term contributes
only to the diagonal entries of $A$. The off-diagonal entry $A_{ij}$ ($i\neq j$)
is nonzero when documents $i$ and $j$ are joined by an edge ($\ell_{ij}=\pm1$). Specifically, from Eq.~\eqref{eq:A}, each pairwise relation term adds $\pm 2w_{ij}$ at positions $(i,j)$ and $(j,i)$ of $A$. The sparsity
pattern of $A$ therefore coincides with the adjacency structure of
$G$. 
For a symmetric positive
definite matrix with this property, the Demko--Moss--Smith decay
bound~\citep{demko1984decay,benzi2007decay} gives an entry-wise
decay of $A^{-1}$ that is exponential in graph distance:
\begin{equation}
\label{eq:a_ij}
\lvert (A^{-1})_{ij}\rvert \;\le\; \gamma q^{\,d_G(i,j)},
\qquad
q = \frac{\sqrt{\kappa}-1}{\sqrt{\kappa}+1},
\end{equation}
where $\gamma=(1+\sqrt{\kappa})^2/(2\sigma_{\max})=O(1/\sigma_{\min})$.

\paragraph{Step 3: Superposition bound.} Combining~Eqs.~\eqref{eq:b_f_k}-\eqref{eq:a_ij} and 
using the triangle inequality, we obtain
\begin{align*}
\lvert\Delta\tau_i\rvert
&\le \sum_{k=1}^{m} \left|(A^{-1})_{i f_k}\right| \cdot \left|\Delta b_{f_k}\right| \\
&\le \sum_{k=1}^{m}2\gamma\lambda_f\lvert\Delta y_{f_k}\rvert  q^{d_G(i,f_k)}.
\end{align*}

\subsection{Proof of Lemma~\ref{thm:convergence} (Convergence)}
\label{proof2}
Let
$
\widetilde{\boldsymbol{\tau}}^{(t)}
=
\boldsymbol{\tau}^{(t-1)}
-
\eta\nabla\mathcal{L}\!\left(\boldsymbol{\tau}^{(t-1)}\right)
$
denote the gradient descent step before projection, and let
$
\boldsymbol{\tau}^{(t)}
=
P_{[0,1]^N}\!\left(\widetilde{\boldsymbol{\tau}}^{(t)}\right)
$
denote the projected iterate. 
Since $\mathcal{L}$ is quadratic with Hessian $A$, for any
$\boldsymbol{\tau}$ we have $\nabla\mathcal{L}(\boldsymbol{\tau}) = A(\boldsymbol{\tau} - \boldsymbol{\tau}^*)$. Therefore,
\begin{equation}
\label{eq:iteration}
\begin{aligned}
\widetilde{\boldsymbol{\tau}}^{(t)}-\boldsymbol{\tau}^*
&=
\boldsymbol{\tau}^{(t-1)}
-\eta\nabla\mathcal{L}\!\left(\boldsymbol{\tau}^{(t-1)}\right)
-\boldsymbol{\tau}^* \\
&=
\boldsymbol{\tau}^{(t-1)}-\boldsymbol{\tau}^*
-\eta A\left(\boldsymbol{\tau}^{(t-1)}-\boldsymbol{\tau}^*\right) \\
&=
(I-\eta A)\left(\boldsymbol{\tau}^{(t-1)}-\boldsymbol{\tau}^*\right).
\end{aligned}
\end{equation}
With the fixed step size $\eta^\star = 2/(\sigma_{\min} + \sigma_{\max})$,
each eigenvalue $\sigma$ of $A$ is mapped to an eigenvalue
$1-\eta^\star\sigma$ of $I-\eta^\star A$. Since
$\sigma\in[\sigma_{\min},\sigma_{\max}]$, these eigenvalues lie in
$
\left[
-\frac{\kappa-1}{\kappa+1},
\frac{\kappa-1}{\kappa+1}
\right].
$
Thus, the largest absolute eigenvalue of $I-\eta^\star A$ is
$
\rho=\frac{\kappa-1}{\kappa+1}.
$
Because $I-\eta^\star A$ is symmetric, its spectral norm is equal to its
largest absolute eigenvalue, so
\[
\left\|I-\eta^\star A\right\|_2=\rho.
\]
Taking norms in \eqref{eq:iteration} then gives
\[
\left\|\widetilde{\boldsymbol{\tau}}^{(t)}-\boldsymbol{\tau}^*\right\|
\le
\rho
\left\|\boldsymbol{\tau}^{(t-1)}-\boldsymbol{\tau}^*\right\|.
\]

The projection operator $P_{[0,1]^N}(\cdot)$ is non-expansive, and hence
\begin{align*}
\lVert\boldsymbol{\tau}^{(t)} - \boldsymbol{\tau}^*\rVert
&\le \lVert\widetilde{\boldsymbol{\tau}}^{(t)} - \boldsymbol{\tau}^*\rVert \nonumber\le \rho\,\lVert\boldsymbol{\tau}^{(t-1)} - \boldsymbol{\tau}^*\rVert.
\end{align*}

\section{Impact of Mislabeled Edges}

A mislabeled edge between documents $i$ and $j$ changes the Hessian $A$ in Eq.~\eqref{eq:A} only at the entries involving $i$ and $j$. Under the same assumptions used in Lemma~\ref{thm:localization}, the entry-wise decay bound on $A^{-1}$ from~\citet{demko1984decay} implies that the resulting perturbation to the optimized trust scores decreases exponentially with the shortest-path distance from the mislabeled edge. Therefore, the influence of an edge error diminishes rapidly as it propagates to more distant documents. 

\section{Prompt Templates}
\label{app:prompts}

\paragraph{Trust-aware prompt.}
Each retrieved document is annotated with both its trust score $\tau_i$ and a trust label. We assign the label based on the trust score: high if $\tau_i \geq 0.7$, low if $\tau_i \leq 0.3$, and medium otherwise.
The system instruction is:

\begin{promptbox}
You are a QA system. Each document has a trust score.
Strongly prefer HIGH-trust documents. Ignore or discount LOW-trust
documents, they may contain deliberate misinformation.
Give a short, direct answer.
\end{promptbox}

\noindent Documents are formatted as:

\begin{promptbox}

[Document 1] (Trust: HIGH, score=0.92): <text>

[Document 2] (Trust: LOW, score=0.15): <text>

...
\end{promptbox}

\noindent\textbf{Baseline prompt.}
VanillaRAG uses the following prompt. The other baselines use the prompts specified in their original papers.

\begin{promptbox}
Answer the following question based on the provided documents.
Give a short, direct answer.

Documents:
[Document 1]: <text>
...

Question: <query>

Short answer:
\end{promptbox}

\section{Quality of NLI-derived Edges}
\label{app:nli_quality}

We evaluate the NLI-derived edges against the known factual and contradictory structure of the corpus on the three datasets. We treat an edge between a factual document and a contradictory document as a ground-truth contradictory relation, and an edge between two documents of the same category as a supportive relation. Table~\ref{tab:nli_quality} reports the precision and recall of the contradictory and supportive edges produced by the NLI model.

\begin{table}[t]
\centering
\small
\setlength{\tabcolsep}{4pt}
\begin{tabular}{@{}lcccc@{}}
\toprule
\multirow{2}{*}{Dataset}
 & \multicolumn{2}{c}{Contradictory}
 & \multicolumn{2}{c}{Supportive} \\
\cmidrule(lr){2-3} \cmidrule(lr){4-5}
 & Precision & Recall & Precision & Recall \\
\midrule
MS~MARCO & 0.65 & 0.18 & 0.71 & 0.14 \\
NQ       & 0.69 & 0.21 & 0.68 & 0.17 \\
TriviaQA & 0.58 & 0.13 & 0.62 & 0.12 \\
\bottomrule
\end{tabular}
\caption{Precision and recall of the NLI-derived contradictory and supportive edges.}
\label{tab:nli_quality}
\end{table}

The NLI-derived edges have moderate precision and low recall, as the NLI model predicts most document pairs as neutral. The low recall indicates that many relations are omitted, resulting in a sparse graph, which reflects practical settings where only a limited subset of inter-document relations can be identified. The imperfect precision indicates that the graph also contains some incorrectly labeled edges. Nevertheless, the NLI-only ablation in Table~\ref{tab:fb_baseline} shows that \name still outperforms the baselines when relying only on these imperfect edges, which suggests that trust propagation is robust to noise and sparsity in graph construction.

\section{Performance on QACC}

We evaluate \name on naturally occurring conflicts without any synthetic injection. We used QACC~\cite{liu2025open}, a human-annotated open-domain QA dataset in which unambiguous questions are paired with real web contexts retrieved via Google Search. The conflicts among these contexts arise naturally on the web rather than from injected documents. We randomly sampled 100 questions from QACC and constructed the evaluation corpus from their retrieved contexts. As shown in Table~\ref{tab:qacc}, TrustPropRAG achieves the best EM (0.59) on this naturally conflicting corpus, outperforming the strongest baseline ReliabilityRAG (0.56) and improving over VanillaRAG by 0.11. These results indicate that the effectiveness of TrustPropRAG is not limited to synthetically injected contradictions and extends to naturally occurring conflicts.

\begin{table}[t]
\centering
\small
\setlength{\tabcolsep}{4pt}
\begin{tabular}{@{}lc@{}}
\toprule
Method & EM \\
\midrule
VanillaRAG & 0.48 \\
InstructRAG & 0.54 \\
AstuteRAG & 0.50 \\
ReliabilityRAG & 0.56 \\
TrustRAG\textsubscript{stage 1} & 0.50 \\
\name & \textbf{0.59} \\
\bottomrule
\end{tabular}
\caption{EM on QACC (GPT-4o-mini, MiniLM, $k{=}5$, FB=30\%).}
\label{tab:qacc}
\end{table}

\end{document}